\documentclass[letterpaper, 10 pt, conference]{ieeeconf}

\IEEEoverridecommandlockouts
\usepackage{amsmath,amssymb,mathtools,bm}
\usepackage{booktabs}
\usepackage{multirow}
\usepackage{array}
\usepackage{graphicx}
\usepackage{xcolor}
\usepackage{url}
\usepackage[ruled]{algorithm2e}
\usepackage{flushend}

\title{
Safe Meta-Reinforcement Learning via Information Space Reachability
}

\author{Zeyang Li, Sunbochen Tang, Navid Azizan%
\thanks{The authors are with the Laboratory for Information and Decision Systems (LIDS), Massachusetts Institute of Technology, Cambridge, MA 02139, USA. Emails: {\tt\small \{zeyang, tangsun, azizan\}@mit.edu}}
}

\newtheorem{theorem}{Theorem}
\newtheorem{proposition}{Proposition}

\newtheorem{definition}{Definition}

\newcommand{\ind}{\mathbf{1}}

\newcommand{\algname}{ISDAC}

\begin{document}

\maketitle
\thispagestyle{empty}
\pagestyle{empty}

\begin{abstract}
Meta-reinforcement learning (meta-RL) enables agents to adapt to unseen tasks with limited experience. Despite its promise, the application of meta-RL in real-world tasks is hindered by safety requirements, which have been underexplored in prior work. In this paper, we propose a safe meta-RL framework that explicitly accounts for safety during adaptation. Our key insight is to reason about safety in the information space, which captures both the physical state and the agent’s belief over the underlying task. Within this space, we introduce a safety value function that measures the probability of the agent avoiding unsafe regions indefinitely. We show that this function satisfies a self-consistency condition and a Bellman equation, which make it learnable via meta-RL. Based on this formulation, we develop a safe meta-RL algorithm that learns the safety value function and leverages it for safety filtering and constrained policy optimization. Experiments on meta-RL benchmarks demonstrate the effectiveness of the proposed method.
\end{abstract}

\section{Introduction}

Meta-reinforcement learning (meta-RL) \cite{beck2025tutorial} is a powerful paradigm for learning to learn in sequential decision-making problems. It enables agents to adapt quickly to previously unseen tasks using only limited data \cite{clavera2018learning, yu2020meta}. Specifically, an agent is trained over a distribution of tasks and learns an effective adaptation strategy, namely a meta-policy, that can be transferred to new tasks drawn from a similar distribution. This allows the agent to adapt to a new task with only a small amount of experience, significantly improving efficiency compared with learning a new policy from scratch.

Existing meta-RL methods can be broadly categorized according to how they achieve rapid adaptation. Gradient-based approaches, exemplified by MAML \cite{finn2017model}, formulate learning to learn as a bilevel optimization problem, in which the meta-learner acquires an initialization that can be quickly adapted to a new task through a small number of gradient updates. More recently, latent-belief approaches such as PEARL \cite{rakelly2019efficient} and VariBAD \cite{zintgraf2021varibad} adopt a Bayes-adaptive perspective, inferring a latent representation of the underlying task from observed experience and conditioning the policy on this inferred representation. By explicitly modeling task uncertainty, these methods enable efficient adaptation without requiring online gradient updates, and have demonstrated strong empirical performance across a range of meta-RL benchmarks.

Despite the widely recognized effectiveness of meta-RL, its application to real-world decision-making problems remains limited, largely due to safety requirements. In many practical settings, success depends not only on reward maximization but also on maintaining safety; otherwise, catastrophic failures may occur. This motivates the study of \emph{safe} meta-RL, which aims to learn a safe meta-policy that can adapt to new tasks such that the resulting policy not only achieves task objectives, such as reward maximization, but also satisfies safety constraints.

Existing safe meta-RL methods have notable limitations. First, most existing works adopt the constrained Markov decision process (CMDP) framework \cite{altman1999constrained}, in which the constraint requires the expected trajectory cost to remain below a predefined threshold. This formulation is often ill-suited to real-world applications, since constraints are enforced only in expectation, allowing the agent to violate safety requirements at particular states. Khattar et al. \cite{khattar2023a} formulate safe meta-RL in a CMDP-within-online framework and establish task-averaged guarantees for reward maximization and constraint satisfaction. Cho and Sun \cite{cho2024constrained} employ successive convex-constrained policy updates across multiple tasks via differentiable convex programming. Xu and Zhu \cite{xu2025efficient} propose safe policy adaptation and safe meta-policy training, and establish an anytime safety guarantee for policy adaptation. Second, these approaches largely inherit the optimization-centric perspective of methods such as MAML \cite{finn2017model}, where both learning and adaptation take place in the policy parameter space. Although this viewpoint is conceptually clean and theoretically appealing, it has been shown to be inefficient for complex control tasks \cite{rakelly2019efficient, zintgraf2021varibad}, due to expensive inner-loop optimization, on-policy data requirements, and the inability to explicitly capture task uncertainty, which is only indirectly reflected in policy parameters.

In this paper, we propose a novel safe meta-RL framework that explicitly addresses these limitations. First, we consider the state-wise constraint formulation, which requires the agent to satisfy safety constraints at every state it visits \cite{zhao2023state}. Second, we adopt the Bayes-adaptive perspective in meta-RL \cite{rakelly2019efficient, zintgraf2021varibad} and reason about safety in the information space, rather than in the policy parameter space. The information space is defined by the physical state together with the posterior belief over tasks. We then formulate safety preservation as a reachability problem in this space. This viewpoint offers both conceptual and practical advantages. It makes explicit that safety under task uncertainty is inherently belief-dependent: the same physical state may admit different safety guarantees under different task posteriors. By learning safety-preserving behavior directly in information space, the agent can better handle task uncertainty, leading to policies that are both safe and performant on meta-RL tasks.
The main contributions of this paper are summarized as follows.
\begin{itemize}
    \item We formulate meta-RL for safety preservation as a Bayes-adaptive reachability problem in information space. By introducing safety value functions and establishing their self-consistency conditions and Bellman equations, we develop a meta-RL framework for state-wise safety.
    \item We propose a practical safe meta-RL algorithm for complex high-dimensional systems, in which value functions and policies are approximated with neural networks, and belief updates are approximated by a neural encoder that infers latent task representations from online interactions.
    \item We demonstrate the effectiveness of the proposed algorithm on widely used meta-RL benchmarks.
\end{itemize}

\section{Problem Statement}

Consider a family of Markov decision processes (MDPs),
\begin{equation}
\nonumber
\mathcal{M}_z=\left(\mathcal{X},\mathcal{U},P_z,r_z,h_z,\gamma\right),
\end{equation}
parameterized by a task descriptor \(z\in\mathcal{Z}\). Here, \(\mathcal{Z}\) denotes the task space, \(\mathcal{X}\) denotes the state space, and \(\mathcal{U}\) denotes the action space. For each task \(z\), \(P_z(\cdot| x,u)\) defines the transition kernel on \(\mathcal{X}\), \(r_z:\mathcal{X}\times\mathcal{U}\to\mathbb{R}\) defines the reward function, and \(h_z:\mathcal{X}\to\mathbb{R}\) defines the constraint function. \(\gamma\in(0,1)\) is the reward discount factor.

We consider a meta-reinforcement learning (meta-RL) setting, where a task \(z\in\mathcal{Z}\) is sampled from a prior distribution \(p(z)\) at the beginning of each episode and then remains fixed for the duration of that episode. For clarity, we assume throughout the theoretical development that \(\mathcal{Z}\), \(\mathcal{X}\), and \(\mathcal{U}\) are finite. The extension to continuous spaces is standard.

Given a sampled task \(z\), the environment evolves according to the transition kernel \(x_{t+1}\sim P_z(\cdot| x_t,u_t)\), with instantaneous reward \(r_t=r_z(x_t,u_t)\) and constraint value \(h_t=h_z(x_t)\). The goal of meta-RL is to leverage experience collected across training tasks drawn from \(p(z)\) in order to learn a controller that can rapidly adapt to a new task using online interaction. Concretely, at each decision time \(t\), the policy may depend on the interaction history observed so far. Let
\begin{equation}
\nonumber
c_t=(x_0,h_0,u_0,r_0,\cdots,x_{t-1},h_{t-1},u_{t-1},r_{t-1},x_t,h_t)
\end{equation}
denote the context available up to time \(t\). A history-dependent policy \(\pi\) then selects actions according to \(c_t\). The standard meta-RL objective is to maximize the task-averaged discounted cumulative reward
\begin{equation}
\nonumber
J(\pi)
=
\mathbb{E}_{z\sim p(z),\,\pi}
\left[
\sum_{t=0}^{\infty}\gamma^t r_z(x_t,u_t)
\right].
\label{eq:meta-rl-objective}
\end{equation}
In this paper, we additionally require safety, in the sense that the state constraints \(h_z(x_t)\ge 0\) are satisfied for all \(t\) with high probability under the task distribution and the policy.

At the theoretical level, meta-RL can be formulated exactly as a Bayes-adaptive control problem \cite{duff2002optimal, ghavamzadeh2015bayesian}. Define the posterior belief over tasks by \(b_t(z)=\mathbb{P}(z| c_t)\), which we refer to as the belief state. Here, \(b_t\in\Delta(\mathcal{Z})\) is a probability distribution over the task space \(\mathcal{Z}\). The information state is defined as the pair consisting of the physical state and the belief state:
\begin{equation}
\nonumber
s_t=(x_t,b_t)\in\mathcal{S}=\mathcal{X}\times\Delta(\mathcal{Z}).
\end{equation}
The information state \(s_t\) constitutes an exact sufficient statistic for decision-making under task uncertainty: there exists a deterministic Markov policy on \(\mathcal{S}\) that achieves optimality.

Let \(\mathcal{O}\) denote the post-action observation space, and define the post-action observation at time \(t\) by
\begin{equation}
\nonumber
o_{t+1}=(r_t,x_{t+1},h_{t+1})\in\mathcal{O}.
\end{equation}
The task-conditioned observation kernel is
\begin{equation}
\begin{aligned}
&O_z(r,x^+,h^+| x,u)
\\&=
P_z(x^+| x,u)\,
\ind\{r=r_z(x,u)\}\,
\ind\{h^+=h_z(x^+)\}.
\end{aligned}
\label{eq:task-observation-kernel}
\end{equation}
The exact posterior update is then given by Bayes' rule
\begin{equation}
\nonumber
b_{t+1}(z)=
\frac{
O_z(o_{t+1}| x_t,u_t)\,b_t(z)
}{
\sum_{\bar z\in\mathcal Z}
O_{\bar z}(o_{t+1}| x_t,u_t)\,b_t(\bar z)
}.
\end{equation}

Although the Bayes-adaptive formulation provides a conceptually clean characterization of the meta-RL problem, maintaining the exact posterior \(b_t\) is generally intractable in complex settings. As a result, an important class of modern meta-RL algorithms approximates the posterior update via amortized inference from trajectory context. Specifically, one introduces an encoder \(q_\phi(\xi| c_t)\) and uses a finite-dimensional parameterization as a proxy for the belief \(b_t\). In this paper, we adopt this perspective: the exact safety analysis is carried out in the information state \((x_t,b_t)\), whereas the practical method replaces \(b_t\) with a learned latent representation inferred online from context.

\section{Reachability Analysis on Information Space}

In this section, we develop a meta-RL framework to characterize and optimize safety directly over the information space \(\mathcal{S}=\mathcal{X}\times\Delta(\mathcal{Z})\). As discussed earlier, it is necessary to reason on the information state \(s=(x,b)\), rather than the physical state \(x\) alone, since the belief \(b\) encodes epistemic uncertainty about the underlying task and therefore fundamentally affects the safety guarantees that can be established.

Define the task-dependent safety indicator as \(g_z(x)=\ind\{h_z(x)\ge 0\}\) and the corresponding information-state safety indicator as $g(s)=\sum_{z\in\mathcal{Z}} b(z)\,g_z(x)$.
Since the current constraint value \(h_t\) is part of the context \(c_t\), every task in the support of \(b_t\) is consistent with the observed current safety label. Therefore \(g(s_t)\in\{0,1\}\) for every reachable information state. Accordingly, all statements below are on the reachable part of \(\mathcal{S}\).

We consider stochastic policies \(\pi:\mathcal{S}\to\Delta(\mathcal{U})\). For \(s=(x,b)\in\mathcal S\), define the predictive observation kernel by
\begin{equation}
\nonumber
O(o^+| s,u)
=
\sum_{z\in\mathcal Z} b(z)\,O_z(o^+| x,u),
\end{equation}
where \(O_z\) is the task-conditioned observation kernel in \eqref{eq:task-observation-kernel}. The Bayesian update map is
\begin{equation}
\nonumber
\Phi(b,x,u,o^+)(z)
=
\frac{
O_z(o^+| x,u)\,b(z)
}{
\sum_{\bar z\in\mathcal Z} O_{\bar z}(o^+| x,u)\,b(\bar z)
}.
\end{equation}
The corresponding next information state is
\begin{equation}
\nonumber
F(s,u,o^+)=(x^+,\Phi(b,x,u,o^+)),
\end{equation}
where \(o^+=(r,x^+,h^+)\).

Fix \(s=(x,b)\in\mathcal{S}\) and \(u\in\mathcal{U}\). Under a stochastic policy \(\pi\), let \(\mathbb{P}_{s,u}^{\pi}\) and \(\mathbb{E}_{s,u}^{\pi}\) denote the probability law and expectation of the controlled process defined as follows: a task descriptor \(z\sim b\) is drawn once and kept fixed; \(s_0=s\) and \(u_0=u\); conditional on \(z,x_t,u_t\), the post-action observation \(o_{t+1}\) is sampled according to \(O_z(\cdot| x_t,u_t)\); the next information state is \(s_{t+1}=F(s_t,u_t,o_{t+1})\); and for all \(t\ge 1\), the action \(u_t\) is sampled from \(\pi(\cdot| s_t)\).

Define the first-violation time by
\begin{equation}
\nonumber
\tau=\inf\{t\ge 0:\ h_z(x_t)<0\}
=
\inf\{t\ge 0:\ g(s_t)=0\},
\end{equation}
with the convention \(\tau=\infty\) if the trajectory remains safe forever.
The quantity of primary interest is the probability that the process never reaches the unsafe region. This leads naturally to the following safety value functions.

\begin{definition}[safety value functions]
\label{def:safety_value}
For a stochastic policy \(\pi\), define its safety action-value function by
\begin{equation}
Q_h^{\pi}(s,u)
=
\mathbb{P}_{s,u}^{\pi}(\tau=\infty)
=
\mathbb{E}_{s,u}^{\pi}
\left[
\prod_{t=0}^{\infty} g(s_t)
\right].
\label{eq:undiscounted-safety-q}
\end{equation}
The associated safety state-value function is
\begin{equation}
\nonumber
V_h^{\pi}(s)
=
\sum_{u\in\mathcal{U}} \pi(u| s)\,Q_h^{\pi}(s,u).
\end{equation}

The optimal safety values are defined by \(Q_h^*(s,u)=\sup_{\pi} Q_h^{\pi}(s,u)\) and \(V_h^*(s)=\sup_{\pi} V_h^{\pi}(s)\).
\end{definition}

Equation \eqref{eq:undiscounted-safety-q} is the probability, under the posterior belief \(b\), that safety is maintained for all time after taking action \(u\) at information state \(s\) and then following \(\pi\). By maximizing this quantity, we obtain policies that are optimal for rendering the system safe, together with the corresponding optimal values.

The safety value function admits a self-consistency condition: once the current safety label and the next observation are revealed, the remainder of the problem has exactly the same form as the original one. This idea is inspired by Hamilton--Jacobi reachability \cite{bansal2017hamilton} and its RL formulations \cite{fisac2019bridging, li2024safe}.

\begin{proposition}[safety self-consistency condition]
\label{prop:undiscounted-self-consistency}
For every stochastic policy \(\pi\) and every \((s,u)\in\mathcal{S}\times\mathcal{U}\),
\begin{equation}
Q_h^{\pi}(s,u)
=
g(s)\sum_{o^+\in\mathcal{O}}
O(o^+| s,u)\,
V_h^{\pi}(F(s,u,o^+)).
\label{eq:undiscounted-self-consistency}
\end{equation}
\end{proposition}

\begin{proof}
By definition,
\begin{equation}
\nonumber
Q_h^{\pi}(s,u)
=
\mathbb{E}_{s,u}^{\pi}
\left[
\prod_{t=0}^{\infty} g(s_t)
\right]
=
g(s)\,
\mathbb{E}_{s,u}^{\pi}
\left[
\prod_{t=1}^{\infty} g(s_t)
\right].
\end{equation}
Conditioning on the next observation \(o_1=o^+\) gives
\begin{equation}
\nonumber
\begin{aligned}
&Q_h^{\pi}(s,u)
\\&=
g(s)\sum_{o^+\in\mathcal{O}}
O(o^+| s,u)\,
\mathbb{E}_{s,u}^{\pi}
\left[
\prod_{t=1}^{\infty} g(s_t)
\ \middle|\
o_1=o^+
\right].
\end{aligned}
\end{equation}
Once \(o_1=o^+\) is revealed, the next information state is \(s_1=F(s,u,o^+)\), and the future control law is again \(\pi\), so
\begin{equation}
\nonumber
\mathbb{E}_{s,u}^{\pi}
\left[
\prod_{t=1}^{\infty} g(s_t)
\ \middle|\
o_1=o^+
\right]
=
V_h^{\pi}(F(s,u,o^+)).
\end{equation}
Substituting this identity proves \eqref{eq:undiscounted-self-consistency}.
\end{proof}

The self-consistency condition immediately suggests a Bellman-type characterization for the optimal safety values. The next theorem shows that the optimal perpetual-safety probability can indeed be computed recursively on information space.

\begin{theorem}[safety Bellman equation]
\label{thm:undiscounted-bellman}
The optimal safety values satisfy the following recursive characterization:
\begin{equation}
\label{eq:undiscounted-Bellman}
\left\{\begin{array}{l}
Q_h^*(s,u)=g(s)\sum_{o^{+}\in \mathcal{O}} O\left(o^{+} | s, u\right) V_h^*\left(F(s,u,o^+)\right) \\
V_h^*(s)=\max _{u \in \mathcal{U}} Q_h^*(s, u).\\
\end{array}\right.
\end{equation}
\end{theorem}

\begin{proof}
Fix \((s,u)\in\mathcal{S}\times\mathcal{U}\). By Proposition \ref{prop:undiscounted-self-consistency}, for every stochastic policy \(\pi\),
\begin{equation}
\nonumber
Q_h^{\pi}(s,u)
=
g(s)\sum_{o^+\in\mathcal{O}}
O(o^+| s,u)\,
V_h^{\pi}(F(s,u,o^+)).
\end{equation}
Since \(V_h^{\pi}(F(s,u,o^+))\le V_h^*(F(s,u,o^+))\) for every \(o^+\), we obtain
\begin{equation}
Q_h^*(s,u)
\le
g(s)\sum_{o^+\in\mathcal{O}}
O(o^+| s,u)\,
V_h^*(F(s,u,o^+)).
\label{eq:undiscounted-proof-upper}
\end{equation}

Conversely, conditioned on the first post-action observation \(o^+\), the remaining problem is again the same safety-maximization problem started from the successor information state \(F(s,u,o^+)\). By Bellman's principle of optimality for the information-state controlled Markov process, an optimal continuation policy must therefore attain value \(V_h^*(F(s,u,o^+))\) at each reachable successor state. Hence
\begin{equation}
Q_h^*(s,u)
\ge
g(s)\sum_{o^+\in\mathcal{O}}
O(o^+| s,u)\,
V_h^*(F(s,u,o^+)).
\label{eq:undiscounted-proof-lower}
\end{equation}
Combining \eqref{eq:undiscounted-proof-upper} and \eqref{eq:undiscounted-proof-lower} proves the first equality in \eqref{eq:undiscounted-Bellman}.

For the state-value function, the first decision at state \(s\) is the choice of an action distribution \(\mu\in\Delta(\mathcal U)\). Therefore
\begin{equation}
\nonumber
V_h^*(s)
=
\sup_{\mu\in\Delta(\mathcal U)}
\sum_{u\in\mathcal U}\mu(u)\,Q_h^*(s,u).
\end{equation}
Since the right-hand side is linear in \(\mu\), its supremum over the simplex \(\Delta(\mathcal U)\) is attained at an extreme point, i.e., at a deterministic action. Hence $V_h^*(s)=\max_{u\in\mathcal U}Q_h^*(s,u)$.
\end{proof}

The proposed safety value function \eqref{eq:undiscounted-safety-q} characterizes perpetual safety. However, the corresponding self-consistency condition \eqref{eq:undiscounted-self-consistency} and Bellman equation \eqref{eq:undiscounted-Bellman} generally do not induce contraction mappings. We therefore introduce a discounted surrogate that retains a clear probabilistic interpretation while yielding contraction mappings and stable fixed-point characterizations.

\begin{definition}[discounted safety value functions]
\label{def:discounted_safety_value}
Fix a safety discount factor \(\gamma_h\in(0,1)\). For a stochastic policy \(\pi\), define the discounted safety action-value function by
\begin{equation}
Q_{h,\gamma_h}^{\pi}(s,u)
=
(1-\gamma_h)\,
\mathbb{E}_{s,u}^{\pi}
\left[
\sum_{t=0}^{\infty}
\gamma_h^t
\prod_{k=0}^{t} g(s_k)
\right].
\label{eq:discounted-safety-q}
\end{equation}
The discounted safety state-value function is
\begin{equation}
\nonumber
V_{h,\gamma_h}^{\pi}(s)
=
\sum_{u\in\mathcal{U}} \pi(u| s)\,Q_{h,\gamma_h}^{\pi}(s,u).
\label{eq:discounted-safety-v}
\end{equation}

The optimal discounted safety values are defined by \(Q_{h,\gamma_h}^*(s,u)=\sup_{\pi} Q_{h,\gamma_h}^{\pi}(s,u)\) and \(V_{h,\gamma_h}^*(s)=\sup_{\pi} V_{h,\gamma_h}^{\pi}(s)\).
\end{definition}

The discounted quantity has an interesting probabilistic interpretation. Suppose that, independently of the task, the episode terminates after each step with probability \(1-\gamma_h\). If \(T\) denotes the resulting geometric horizon, then \(Q_{h,\gamma_h}^{\pi}(s,u)\) is exactly the probability that the trajectory remains safe throughout the episode, i.e., that no safety violation occurs before termination. The next proposition makes this interpretation precise.

\begin{proposition}[geometric-horizon interpretation]
\label{prop:geometric-horizon}
Let \(T\) be independent of the controlled process and geometrically distributed on \(\{0,1,2,\ldots\}\) with \(\mathbb{P}(T=t)=(1-\gamma_h)\gamma_h^t\).
Then, for every \((s,u)\),
\begin{equation}
\nonumber
Q_{h,\gamma_h}^{\pi}(s,u)=\mathbb{P}_{s,u}^{\pi}(\tau>T).
\end{equation}
\end{proposition}

\begin{proof}
For each \(t\ge 0\), we have
\begin{equation}
\nonumber
\prod_{k=0}^{t} g(s_k)=\ind\{\tau>t\}.
\end{equation}
Substituting this identity into \eqref{eq:discounted-safety-q} yields
\begin{align}
\nonumber
Q_{h,\gamma_h}^{\pi}(s,u)
&=
(1-\gamma_h)\sum_{t=0}^{\infty}\gamma_h^t
\mathbb{P}_{s,u}^{\pi}(\tau>t)
\\
&=
\nonumber
\sum_{t=0}^{\infty}
\mathbb{P}(T=t)\,
\mathbb{P}_{s,u}^{\pi}(\tau>t)
\\
&=
\nonumber
\mathbb{P}_{s,u}^{\pi}(\tau>T),
\end{align}
where the last equality uses the independence of \(T\) and the controlled process.
\end{proof}

The discounted analogue of the self-consistency condition has the same basic structure as in the undiscounted case, except that the current safety indicator contributes an immediate term weighted by \(1-\gamma_h\).

\begin{proposition}[discounted self-consistency condition]
For every stochastic policy \(\pi\) and every \((s,u)\in\mathcal{S}\times\mathcal{U}\),
\begin{equation}
\begin{aligned}
&Q_{h,\gamma_h}^{\pi}(s,u)
=
(1-\gamma_h)\,g(s)\\ &
+
\gamma_h\,g(s)\sum_{o^+\in\mathcal{O}}
O(o^+| s,u)\,
V_{h,\gamma_h}^{\pi}(F(s,u,o^+)).
\end{aligned}
\label{eq:discounted-self-consistency}
\end{equation}
\end{proposition}

\begin{proof}
Starting from \eqref{eq:discounted-safety-q},
\begin{equation}
\nonumber
\begin{aligned}
&Q_{h,\gamma_h}^{\pi}(s,u)
=
(1-\gamma_h)\,
\mathbb{E}_{s,u}^{\pi}
\left[
g(s_0)+
\sum_{t=1}^{\infty}
\gamma_h^t
\prod_{k=0}^{t} g(s_k)
\right]
\\
&=
(1-\gamma_h)\,g(s)
+
\gamma_h g(s)\,
(1-\gamma_h)\,
\mathbb{E}_{s,u}^{\pi}
\left[
\sum_{\ell=0}^{\infty}
\gamma_h^{\ell}
\prod_{k=1}^{\ell+1} g(s_k)
\right].
\end{aligned}
\end{equation}
Conditioning on the next observation \(o_1=o^+\) gives
\begin{equation}
\nonumber
\begin{aligned}
&(1-\gamma_h)\,
\mathbb{E}_{s,u}^{\pi}
\left[
\sum_{\ell=0}^{\infty}
\gamma_h^{\ell}
\prod_{k=1}^{\ell+1} g(s_k)
\ \middle|\
o_1=o^+
\right] \\&
=
V_{h,\gamma_h}^{\pi}(F(s,u,o^+)).
\end{aligned}
\end{equation}
Averaging with respect to \(O(o^+| s,u)\) proves \eqref{eq:discounted-self-consistency}.
\end{proof}

We now turn to optimality. As in the undiscounted case, the optimal discounted safety values satisfy a Bellman equation.

\begin{theorem}[discounted safety Bellman equation]
The optimal discounted safety values satisfy the following recursive characterization:
\begin{equation}
\nonumber
\left\{\begin{array}{l}
\begin{aligned}
&Q_{h, \gamma_h}^*(s, u)=(1-\gamma_h)\,g(s) \\&+ \gamma_h\,g(s) \sum_{o^{+} \in \mathcal{O}} O\left(o^{+} | s, u\right) V_{h, \gamma_h}^*\left(F(s,u,o^+)\right) \\
\end{aligned}\\
V_{h, \gamma_h}^*(s)=\max _{u \in \mathcal{U}} Q_{h, \gamma_h}^*(s, u).\\
\end{array}\right.
\end{equation}
\end{theorem}

\begin{proof}
The proof follows a similar idea to that of the undiscounted safety Bellman equation, so we omit it here.
\end{proof}

The discounted formulation is especially convenient since it admits fixed-point characterizations. These operators also motivate the critic and actor updates used in the practical algorithm later.

\begin{definition}[safety operators]
\label{def:safety_operators}
Let \(Q:\mathcal{S}\times\mathcal{U}\to\mathbb{R}\) be a bounded function. Define the safety self-consistency operator as
\begin{equation}
\nonumber
\begin{aligned}
&(\mathcal{T}_{h,\gamma_h}^{\pi}Q)(s,u)
=
(1-\gamma_h)\,g(s)
+
\gamma_h\,g(s)
\sum_{o^+\in\mathcal{O}}
O(o^+| s,u) \\ &
\sum_{u^+\in\mathcal{U}}
\pi(u^+| F(s,u,o^+))\,
Q(F(s,u,o^+),u^+),
\end{aligned}
\end{equation}
and the safety Bellman operator as
\begin{equation}
\begin{aligned}
&(\mathcal{T}_{h,\gamma_h}Q)(s,u)
=
(1-\gamma_h)\,g(s)
+
\gamma_h\,g(s) \\&
\sum_{o^+\in\mathcal{O}}
O(o^+| s,u)
\max_{u^+\in\mathcal{U}}
Q(F(s,u,o^+),u^+).
\end{aligned}
\label{eq:discounted-optimality-operator}
\end{equation}
\end{definition}

Based on these definitions, the discounted self-consistency condition and Bellman equation can be written compactly as \(Q_{h,\gamma_h}^{\pi}=\mathcal{T}_{h,\gamma_h}^{\pi}Q_{h,\gamma_h}^{\pi}\) and \(Q_{h,\gamma_h}^*=\mathcal{T}_{h,\gamma_h}Q_{h,\gamma_h}^*\). The next result shows that these operators are monotone contractions, which guarantees uniqueness of the corresponding fixed points.

\begin{theorem}[monotone contractions]
Let \(Q,\widetilde Q:\mathcal{S}\times\mathcal{U}\to\mathbb{R}\) be bounded functions. Then:
\begin{enumerate}
\item
If \(Q \le \widetilde Q\), then \(\mathcal{T}_{h,\gamma_h}^{\pi}Q\le\mathcal{T}_{h,\gamma_h}^{\pi}\widetilde Q\) and \(\mathcal{T}_{h,\gamma_h}Q\le\mathcal{T}_{h,\gamma_h}\widetilde Q\).
\item
The operators \(\mathcal{T}_{h,\gamma_h}^{\pi}\) and \(\mathcal{T}_{h,\gamma_h}\) are \(\gamma_h\)-contractions in the infinity norm:
\begin{equation}
\left\|
\mathcal{T}_{h,\gamma_h}^{\pi}Q
-
\mathcal{T}_{h,\gamma_h}^{\pi}\widetilde Q
\right\|_{\infty}
\le
\gamma_h
\left\|
Q-\widetilde Q
\right\|_{\infty},
\label{eq:contraction-eval}
\end{equation}
and
\begin{equation}
\left\|
\mathcal{T}_{h,\gamma_h}Q
-
\mathcal{T}_{h,\gamma_h}\widetilde Q
\right\|_{\infty}
\le
\gamma_h
\left\|
Q-\widetilde Q
\right\|_{\infty}.
\label{eq:contraction-opt}
\end{equation}
\end{enumerate}
Consequently, each operator has a unique fixed point. In particular, the unique fixed points are \(Q_{h,\gamma_h}^{\pi}\) for \(\mathcal{T}_{h,\gamma_h}^{\pi}\) and \(Q_{h,\gamma_h}^{*}\) for \(\mathcal{T}_{h,\gamma_h}\).
\end{theorem}

\begin{proof}
Monotonicity is straightforward. We prove contraction below. Fix \((s,u)\in\mathcal{S}\times\mathcal{U}\). For the evaluation operator,
\begin{equation}
\nonumber
\begin{aligned}
&\left|
(\mathcal{T}_{h,\gamma_h}^{\pi}Q)(s,u)
-
(\mathcal{T}_{h,\gamma_h}^{\pi}\widetilde Q)(s,u)
\right|
\\
&=
\gamma_h\,g(s)\left|
\sum_{o^+\in\mathcal{O}}
O(o^+| s,u)
\sum_{u^+\in\mathcal{U}}
\pi(u^+| F(s,u,o^+))
\right.
\\&\qquad\qquad\left.
\cdot
\left(
Q(F(s,u,o^+),u^+)-\widetilde Q(F(s,u,o^+),u^+)
\right)
\right|
\\
&\le
\gamma_h\,g(s)\\&
\sum_{o^+\in\mathcal{O}}
O(o^+| s,u)
\sum_{u^+\in\mathcal{U}}
\pi(u^+| F(s,u,o^+))
\left\|Q-\widetilde Q\right\|_{\infty}
\\
&\le
\gamma_h
\left\|Q-\widetilde Q\right\|_{\infty}.
\end{aligned}
\end{equation}
Taking the supremum over \((s,u)\) proves \eqref{eq:contraction-eval}.

For the optimality operator, we use the elementary inequality
\begin{equation}
\nonumber
\left|
\max_{u^+} a_{u^+}
-
\max_{u^+} \tilde a_{u^+}
\right|
\le
\max_{u^+} |a_{u^+}-\tilde a_{u^+}|.
\end{equation}
Applying this to \eqref{eq:discounted-optimality-operator}, we obtain
\begin{equation}
\nonumber
\begin{aligned}
&\left|
(\mathcal{T}_{h,\gamma_h}Q)(s,u)
-
(\mathcal{T}_{h,\gamma_h}\widetilde Q)(s,u)
\right|
\\
&\le
\gamma_h\,g(s)
\sum_{o^+\in\mathcal{O}}
O(o^+| s,u)\\&\quad
\max_{u^+\in\mathcal{U}}
\left|
Q(F(s,u,o^+),u^+)-\widetilde Q(F(s,u,o^+),u^+)
\right|
\\
&\le
\gamma_h
\left\|Q-\widetilde Q\right\|_{\infty}.
\end{aligned}
\end{equation}
Taking the supremum over \((s,u)\) proves \eqref{eq:contraction-opt}. Banach's fixed-point theorem then gives uniqueness of the fixed points.
\end{proof}

The discounted quantity \(Q_{h,\gamma_h}^{\pi}\) is introduced for algorithmic and analytical convenience, but it remains faithful to the original perpetual-safety objective. The next result shows that, for every fixed policy, the discounted surrogate converges to the original value as \(\gamma_h\rightarrow 1\).

\begin{theorem}[recovery of the undiscounted safety value]
For every stochastic policy \(\pi\) and every \((s,u)\in\mathcal{S}\times\mathcal{U}\),
\begin{equation}
\lim_{\gamma_h\rightarrow 1}
Q_{h,\gamma_h}^{\pi}(s,u)
=
Q_h^{\pi}(s,u).
\label{eq:abelian-limit}
\end{equation}
Consequently, for every \(s\in\mathcal{S}\),
\begin{equation}
\nonumber
\lim_{\gamma_h\rightarrow 1}
V_{h,\gamma_h}^{\pi}(s)
=
V_h^{\pi}(s).
\end{equation}
\end{theorem}

\begin{proof}
Fix a stochastic policy \(\pi\) and an initial pair \((s,u)\). Define
\begin{equation}
\nonumber
a_t
=
\mathbb{P}_{s,u}^{\pi}(\tau>t)
=
\mathbb{E}_{s,u}^{\pi}
\left[
\prod_{k=0}^{t} g(s_k)
\right],
\qquad t\ge 0.
\end{equation}
Since the events \(\{\tau>t\}\) are decreasing in \(t\), the sequence \((a_t)_{t\ge 0}\) is nonincreasing and bounded in \([0,1]\). Moreover, by continuity of probability for decreasing events,
\begin{equation}
\lim_{t\to\infty} a_t
=
\mathbb{P}_{s,u}^{\pi}\!\left(\bigcap_{t=0}^{\infty}\{\tau>t\}\right)
=
\mathbb{P}_{s,u}^{\pi}(\tau=\infty)
=
Q_h^\pi(s,u).
\label{eq:at-limit}
\end{equation}
On the other hand, by \eqref{eq:discounted-safety-q},
\begin{equation}
\nonumber
Q_{h,\gamma_h}^{\pi}(s,u)
=
(1-\gamma_h)\sum_{t=0}^{\infty}\gamma_h^t a_t.
\end{equation}

Let \(a_\infty:=Q_h^\pi(s,u)\). Given \(\varepsilon>0\), choose \(N\) such that \(|a_t-a_\infty|\le \varepsilon\) for all \(t\ge N\), which is possible by \eqref{eq:at-limit}. Then
\begin{equation}
\nonumber
\begin{aligned}
&\left|
Q_{h,\gamma_h}^{\pi}(s,u)-a_\infty
\right|
\\
&=
\left|
(1-\gamma_h)\sum_{t=0}^{\infty}\gamma_h^t(a_t-a_\infty)
\right|
\\
&\le
(1-\gamma_h)\sum_{t=0}^{N-1}\gamma_h^t|a_t-a_\infty|
+
(1-\gamma_h)\sum_{t=N}^{\infty}\gamma_h^t|a_t-a_\infty|
\\
&\le
(1-\gamma_h)\sum_{t=0}^{N-1}|a_t-a_\infty|
+
\varepsilon(1-\gamma_h)\sum_{t=N}^{\infty}\gamma_h^t
\\
&\le
(1-\gamma_h)\sum_{t=0}^{N-1}|a_t-a_\infty|
+
\varepsilon.
\end{aligned}
\end{equation}
The first term on the right tends to zero as \(\gamma_h\to 1\), since it is a finite constant multiplied by \(1-\gamma_h\). Therefore
\begin{equation}
\nonumber
\limsup_{\gamma_h\to 1}
\left|
Q_{h,\gamma_h}^{\pi}(s,u)-Q_h^\pi(s,u)
\right|
\le
\varepsilon.
\end{equation}
Since \(\varepsilon>0\) is arbitrary, \eqref{eq:abelian-limit} follows.
The remaining proof for state-value function is straightforward.
\end{proof}

The results in this section provide a theoretical foundation for designing meta-RL algorithms that satisfy safety constraints. The discounted safety Bellman equation suggests a policy-iteration-style approach to computing the discounted safety value function. In practice, this naturally leads to an actor-critic algorithm, in which the actor and critic serve as function approximators for the policy and safety value function, respectively.

\section{Algorithm Design}

In this section, we present a practical safe meta-RL algorithm that seeks to maximize cumulative reward while satisfying given safety constraints. As noted earlier, exact belief updates are generally intractable. We therefore introduce an encoder $q_\phi\left(\xi | c_t\right)$ that provides a finite-dimensional parameterization serving as a proxy for the belief $b_t$. Building on this representation, we employ two actor-critic pairs. The first, consisting of a safety actor and a safety critic, is derived from the theoretical framework developed in the previous section and is dedicated to safety preservation. The second is a standard actor-critic pair for task performance, namely reward maximization. In this module, the safety critic is further used as a safety filter and in constrained policy optimization for training the actor.

The encoder follows the probabilistic latent-variable design introduced in \cite{rakelly2019efficient}. At time $t$, the context is given by $c_t = (x_0, h_0, u_0, r_0, \dots, x_{t-1}, h_{t-1}, u_{t-1}, r_{t-1}, x_t, h_t)$.
Given this context, the encoder produces an approximate Gaussian posterior over the latent task representation $\xi\sim q_\phi(\xi | c_t) = \mathcal{N}(\mu_t, \Sigma_t)$, where $\phi$ denotes the parameters of the encoder network. Specifically, taking $c_t$ as input, the network outputs the mean $\mu_t$ and covariance $\Sigma_t$ of the latent representation $\xi$.

For both the safety critic and the performance critic, we adopt the double Q-network design \cite{haarnoja2018soft} and use two networks for each. The safety critic networks are denoted by $Q_h(x, u, \xi; \psi_1)$ and $Q_h(x, u, \xi; \psi_2)$, while the performance critic networks are denoted by $Q(x, u, \xi; \omega_1)$ and $Q(x, u, \xi; \omega_2)$.
For a cleaner presentation, we use the shorthand notations $Q_h(x, u, \xi; \psi) = \min_{i \in \{1,2\}} Q_h(x, u, \xi; \psi_i)$ and $Q(x, u, \xi; \omega) = \min_{i \in \{1,2\}} Q(x, u, \xi; \omega_i)$. The same convention applies to the corresponding target networks, whose parameters are denoted by $\hat{\psi}_1$, $\hat{\psi}_2$, $\hat{\omega}_1$, and $\hat{\omega}_2$.
Since the safety value lies in $[0,1]$, we apply a sigmoid activation at the output layer of the safety critic to enforce this range.

We denote the task actor, which is responsible for maximizing task performance, by $\pi(x,\xi;\theta)$, and the safety actor, which is responsible for maximizing the safety value, by $\pi_h(x,\xi;\varphi)$. The task actor is stochastic in order to encourage exploration and improve performance, whereas the safety actor is deterministic to support safety preservation. Accordingly, we write $u \sim \pi(x,\xi;\theta)$ and $u = \pi_h(x,\xi;\varphi)$.

Following the soft actor-critic algorithm \cite{haarnoja2018soft}, the performance critic loss is defined as
\begin{equation}
\nonumber
\mathcal{L}_{Q}(\omega_i)
=
\mathbb{E}_{\substack{(x,u,r,x') \sim \mathcal{B} \\
\xi \sim q_\phi(\xi | c_t)}}
\left[
\left(Q(x,u,\xi;\omega_i) - \hat{Q}\right)^2
\right],
\end{equation}
where $\mathcal{B}$ denotes the replay buffer and the target value is given by
\begin{equation}
\nonumber
\hat{Q} = r + \gamma \left( Q(x',u',\xi;\hat{\omega}) - \alpha \log \pi(u' | x', \xi; \theta) \right),
\end{equation}
with $u' \sim \pi(x',\xi;\theta)$. $\alpha$ denotes the regularization parameter.
Based on Definition \ref{def:safety_operators}, the safety critic loss is defined as
\begin{equation}
\nonumber
\mathcal{L}_{Q_h}(\psi_i)
=
\mathbb{E}_{\substack{(x,u,g,x') \sim \mathcal{B} \\
\xi \sim q_\phi(\xi | c_t)}}
\left[
\left(Q_h(x,u,\xi;\psi_i) - \hat{Q}_h\right)^2
\right],
\end{equation}
where the target safety value is given by
\begin{equation}
\nonumber
\hat{Q}_h = (1-\gamma_h) g + \gamma_h g Q_h(x',u',\xi;\hat{\psi}),
\end{equation}
with $u' = \pi_h(x',\xi;\varphi)$.

The safety actor $\pi_h(x,\xi;\varphi)$ is designed to maximize the safety value. Its loss function is therefore defined as
\begin{equation}
\nonumber
\mathcal{L}_{\pi_h}(\varphi)
=
-\mathbb{E}_{\substack{x \sim \mathcal{B},
\xi \sim q_\phi(\xi | c_t)}}
\left[
Q_h(x,\pi_h(x,\xi;\varphi),\xi;\psi)
\right].
\end{equation}

The task actor aims to maximize reward while satisfying the safety constraint. To this end, we perform constrained policy optimization with the constraint $Q_h(x,u,\xi;\psi) \geq 1-\delta$, where $u \sim \pi(x,\xi;\theta)$ and $\delta>0$ is a small positive constant included for numerical stability. We adopt a primal-dual approach, which alternates between updating the task actor parameters and the Lagrange multiplier $\lambda$. The loss function for the task actor is given by
\begin{equation}
\nonumber
\begin{aligned}
\mathcal{L}_{\pi}(\theta)
=
-&\mathbb{E}_{\substack{x \sim \mathcal{B},
\xi \sim q_\phi(\xi | c_t)}}
\left[
Q(x,u,\xi;\omega)
\right.\\&
\left.+\lambda Q_h(x,u,\xi;\psi)-\alpha \log \pi(u | x,\xi;\theta)
\right],
\end{aligned}
\end{equation}
where $u \sim \pi(x,\xi;\theta)$.
The loss function for the Lagrange multiplier is defined as
\begin{equation}
\nonumber
\mathcal{L}_{\lambda}
=
\mathbb{E}_{\substack{x \sim \mathcal{B},
\xi \sim q_\phi(\xi | c_t)}}
\left[
\lambda \left(Q_h(x,u,\xi;\psi)-(1-\delta)\right)
\right],
\end{equation}
where $u \sim \pi(x,\xi;\theta)$, and the multiplier is constrained to satisfy $\lambda \geq 0$.

In addition to its role in constrained policy optimization, the safety critic is also used as a safety filter. By Definitions \ref{def:safety_value} and \ref{def:discounted_safety_value}, the safety values provide an assessment of how safe it is to take action $u$ under the current information state $(x,\xi)$. Accordingly, if the nominal action proposed by the task actor is deemed unsafe by the safety critic, we override it with the action produced by the safety actor.

Following \cite{rakelly2019efficient}, the encoder is trained jointly with both the safety critic and the performance critic, and is regularized toward a standard Gaussian prior $p(\xi)=\mathcal{N}(0,I)$ to prevent collapse. Its loss function is defined as
\begin{equation}
\nonumber
\mathcal{L}_{q}(\phi)
=
\mathcal{L}_{Q_h}(\psi,\phi)
+
\mathcal{L}_{Q}(\omega,\phi)
+
\beta_{\mathrm{KL}} D_{\mathrm{KL}}\!\left(q_\phi(\xi | c_t)\,\|\,p(\xi)\right),
\end{equation}
where \(\beta_{\mathrm{KL}}\) is a positive weighting coefficient.
Thus, the encoder learns a latent representation that is informative for both reward prediction and safety certification.

We call our algorithm \emph{Information Safety Dual Actor-Critic} (\algname).
The overall method consists of two phases, meta-training and meta-testing, summarized in Algorithms \ref{alg:meta_training} and \ref{alg:meta_testing}, respectively. During meta-training, the agent is exposed to a distribution of tasks and jointly learns the encoder, the safety and performance critics, the safety actor, and the task actor. In this way, the latent variable $\xi$ captures task uncertainty, while the learned policies achieve high return under safety constraints across tasks drawn from the training distribution. Note that, following the design in \cite{rakelly2019efficient}, we maintain two buffers during meta-training: a standard RL replay buffer $\mathcal{B}$ and a smaller context buffer $\mathcal{B}_c$. The latter is introduced because encoder learning benefits from data that is closer to on-policy.
During meta-testing, all learned parameters are kept fixed and the agent is deployed on a previously unseen task. Adaptation then proceeds through the accumulation of context and the corresponding posterior update. Once sufficient context has been collected, the adaptation process is complete, and the agent can resample the latent variable $\xi$ for formal deployment or evaluation.

\begin{algorithm}
\caption{\algname: Meta-Training.}
\label{alg:meta_training}
\KwIn{batch of training tasks $\{z_i\}_{i=1}^{N} \sim p(z)$; network parameters $\theta$, $\varphi$, $\omega_1$, $\omega_2$, $\psi_1$, $\psi_2$, $\phi$, target network parameters $\hat{\omega}_1 \leftarrow \omega_1$, $\hat{\omega}_2 \leftarrow \omega_2$, $\hat{\psi}_1 \leftarrow \psi_1$, $\hat{\psi}_2 \leftarrow \psi_2$, target smoothing coefficient $\tau$, regularization coefficient $\alpha$, learning rate $\eta$, multiplier $\lambda$, safety threshold $\delta$.}

Initialize replay buffers $\mathcal{B}^i$ and $\mathcal{B}_c^i$ for training tasks.

\For{each outer step}{
    \For{each task $z_i$}{
        Initialize context $c^i$;
        
        \For{each system step}{
            Sample $\xi \sim q_\phi(\xi | c^i)$;
            
            Sample nominal action $u \sim \pi(x,\xi;\theta)$;
            
            \If{$Q_h(x,u,\xi;\psi) < 1-\delta$}{
                $u \leftarrow \pi_h(x,\xi;\varphi)$;
            }
            
            Execute $u$. Update $\mathcal{B}^i$ and $\mathcal{B}_c^i$;
            
            Update context $c^i$.
        }
    }

    \For{each inner step}{
        \For {each task $z_i$}{
            Sample context batch $c^i \sim \mathcal{B}_c^i$ and data batch $b^i \sim \mathcal{B}^i$;

            Sample $\xi \sim q_\phi(\xi | c^i)$;

            Compute per-task losses $\mathcal{L}_{Q}^{i}(\omega)$, $\mathcal{L}_{Q_h}^{i}(\psi)$, $\mathcal{L}_{\pi}^{i}(\theta)$, $\mathcal{L}_{\pi_h}^{i}(\varphi)$, $\mathcal{L}^{i}_{\lambda}$, and $\mathcal{L}_{q}^{i}(\phi)$.
        }

        \For {each gradient step}{

        $\omega_j \leftarrow \omega_j -\eta \nabla_{\omega_j} \sum_i \mathcal{L}_{Q}^{i}$, for $j\in \left\{1,2\right\}$;

        $\psi_j \leftarrow \psi_j - \eta \nabla_{\psi_j} \sum_i \mathcal{L}_{Q_h}^{i}$, for $j\in \left\{1,2\right\}$;

        $\theta \leftarrow \theta - \eta \nabla_\theta \sum_i \mathcal{L}_{\pi}^{i}$;

        $\varphi \leftarrow \varphi - \eta \nabla_\varphi \sum_i \mathcal{L}_{\pi_h}^{i}$;
        
        $\lambda \leftarrow \max \left\{\lambda - \eta \nabla_\lambda \sum_i \mathcal{L}_{\lambda}^{i},0\right\}$;

        $\phi \leftarrow \phi - \eta \nabla_\phi \sum_i \mathcal{L}_{q}^{i}$;

        $\hat{\omega}_j \leftarrow \tau \omega_j + (1-\tau)\hat{\omega}_j$, for $j\in \left\{1,2\right\}$;
        
        $\hat{\psi}_j \leftarrow \tau \psi_j + (1-\tau)\hat{\psi}_j$, for $j\in \left\{1,2\right\}$.
        }
    }
}
\end{algorithm}

\begin{algorithm}
\caption{\algname: Meta-Testing.}
\label{alg:meta_testing}
\KwIn{Test task $z_I\sim p(z)$, safety threshold $\delta$.}

Initialize context $c^I$.

\For{each system step}{
        Sample $\xi \sim q_\phi(\xi | c^I)$;
            
        Sample nominal action $u \sim \pi(x,\xi;\theta)$;
            
        \If{$Q_h(x,u,\xi;\psi) < 1-\delta$}{
            $u \leftarrow \pi_h(x,\xi;\varphi)$;
        }
            
        Execute $u$. Gather data;
            
        Accumulate context $c^I$.
}
\end{algorithm}

\section{Experiments}

In this section, we evaluate the proposed algorithm on two continuous-control safe meta-RL benchmarks based on the HalfCheetah robot in the MuJoCo simulator~\cite{todorov2012mujoco}, adapted from standard meta-RL benchmarks~\cite{rakelly2019efficient, zintgraf2021varibad}.

\textbf{HalfCheetah-Fwd-Back.}
This environment consists of two tasks: moving forward ($d=+1$) and backward ($d=-1$). The task-specific reward is $r(x,u) = d \cdot v - 0.05 \|u\|_2^2$,
where $v$ denotes the robot's velocity, $d \in \{-1,+1\}$ specifies the desired direction, and the second term penalizes control effort. The safety constraint requires the velocity magnitude to remain below a prescribed threshold: $|v| \leq v_{\max}$ with $v_{\max}=6.0$. Since the reward encourages high-speed locomotion while the constraint limits velocity, there is an inherent conflict between performance and safety in both tasks.

\textbf{HalfCheetah-Vel.}
In this environment, the robot is required to track a task-specific target velocity $v^\star$ sampled uniformly from $[0,3]$. The reward penalizes deviation from the target: $r(x,u) = -|v - v^\star| - 0.05 \|u\|_2^2$.
The same form of safety constraint is imposed, namely $|v| \leq v_{\max}$, with the threshold $v_{\max}=1.5$. The environment contains 100 training tasks and 30 test tasks, each associated with a distinct target velocity. This setting is particularly challenging since the relationship between the reward and the safety constraint depends on the task: when $v^\star \leq v_{\max}$, the two are naturally aligned, whereas when $v^\star > v_{\max}$, they are in direct conflict. The agent must therefore learn to distinguish among tasks and adapt its behavior accordingly.

\begin{figure}[htbp]
    \centering
    \includegraphics[width=3.5in]{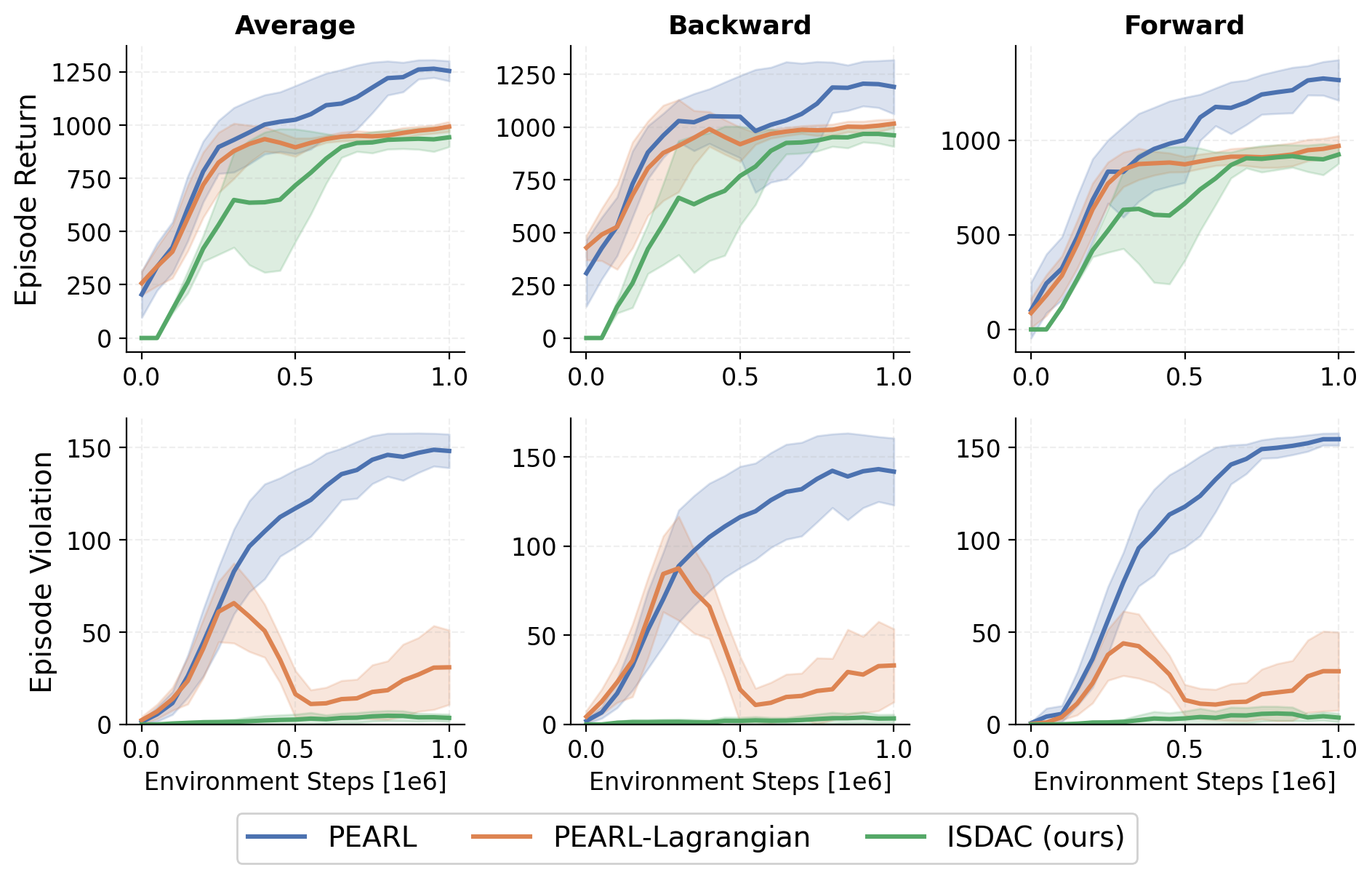}
    \caption{Meta-testing performance on HalfCheetah-Fwd-Back as meta-training progresses. The solid lines correspond to the mean and the shaded regions correspond to $\pm 1$ standard deviation over three seeds.}
    \label{fig:dir}
\end{figure}

\begin{figure*}[htbp]
    \centering
    \includegraphics[width=7.0in]{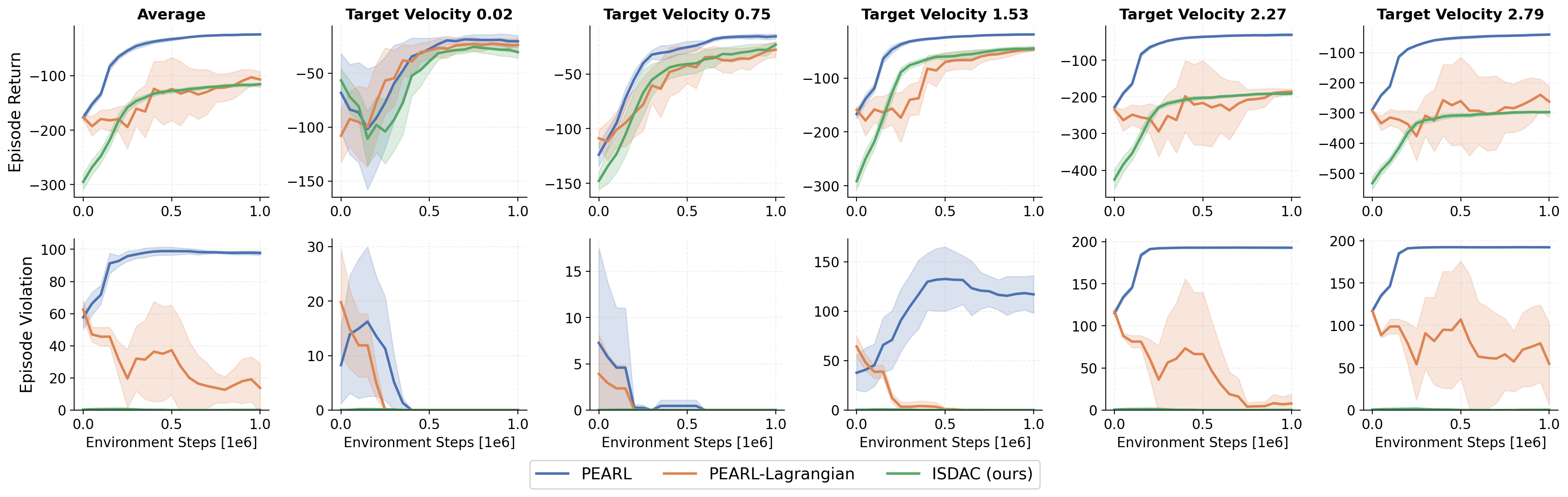}
    \caption{Meta-testing performance on HalfCheetah-Vel as meta-training progresses. The solid lines correspond to the mean and the shaded regions correspond to $\pm 1$ standard deviation over three seeds.}
    \label{fig:vel}
\end{figure*}

\textbf{Baselines.}
We compare our method against two baseline algorithms. The first is PEARL \cite{rakelly2019efficient}, a widely adopted standard meta-RL algorithm. The second, which we call PEARL-Lagrangian, combines PEARL with SAC-Lagrangian \cite{ha2020learning}, a widely used safe RL algorithm under the CMDP framework. This approach enforces trajectory-cost constraints in expectation through a learned cost value function. All three methods use multilayer perceptrons as function approximators and share the same hyperparameter settings. We choose not to include the optimization-based safe meta-RL methods \cite{khattar2023a, cho2024constrained, xu2025efficient}, as we were unable to obtain satisfactory performance within our interaction budget, which may reflect limitations in their sample efficiency.

During meta-training, we report two meta-testing metrics: episode return (cumulative reward) and episode violation (the number of unsafe timesteps). Fig.~\ref{fig:dir} shows the results on HalfCheetah-Fwd-Back, with the left column reporting averages across both tasks and the remaining columns showing the backward and forward tasks separately. Fig.~\ref{fig:vel} shows the results on HalfCheetah-Vel, with the left column reporting the average over all 30 test tasks and the remaining columns showing five representative tasks spanning target velocities in $[0,3]$, from the aligned regime (low $v^\star$) to the conflicting regime (high $v^\star$).

PEARL achieves high return but incurs severe constraint violations, as it lacks any explicit safety mechanism. PEARL-Lagrangian reduces the number of violations, but fails to drive them close to zero, suggesting that a naive application of CMDP-based techniques to meta-RL is insufficient. In contrast, ISDAC maintains near-zero violations while achieving competitive return in both environments.
The results on HalfCheetah-Vel further highlight the adaptivity of ISDAC. For low-velocity tasks, where the target velocity lies within the safety threshold, ISDAC performs comparably to PEARL. As the target velocity exceeds $v_{\max}$, ISDAC consistently maintains a low violation rate by learning to trade off return against safety. These results demonstrate that ISDAC can infer the task-dependent tension between reward maximization and safety from limited experience, and adapt its policy accordingly.

\section{Conclusion}

In this paper, we presented a safe meta-RL framework for learning a meta-policy that can rapidly adapt to unseen tasks while satisfying safety requirements. We introduced the safety value function and established its theoretical properties. Based on these results, we further developed a practical safe meta-RL algorithm for complex high-dimensional tasks, in which the safety value function is learned and used to guide constrained policy optimization.

\bibliographystyle{IEEEtran}
\bibliography{reference}

\end{document}